\documentclass[letterpaper, 10pt, conference]{IEEEtran}  

\usepackage{graphicx}
\usepackage{subfigure} 
\usepackage{bm}
\usepackage{amsfonts,amssymb}
\usepackage{enumerate}
\usepackage{amsmath}

\newtheorem{theorem}{Theorem}

\newenvironment{proof}{{\emph{Proof:}}}{$\hfill\blacksquare$}
\usepackage{balance}
\begin{document}
\title{\LARGE 
Latency Analysis of Consortium Blockchained Federated Learning}


\author{Pengcheng Ren and Tongjiang Yan}
\maketitle
\renewcommand{\thefootnote}{\fnsymbol{footnote}}
\footnotetext{
This work was supported by Fundamental Research Funds for the Central Universities (20CX05012A), the Major Scientific and Technological Projects of CNPC under Grant(ZD2019-183-008), and Shandong Provincial Natural Science Foundation of China (ZR2019MF070). (Corresponding author: Tongjiang Yan.)

The authors are with College of Science, China University of Petroleum, Qingdao 266555, China (email: z19090012@s.upc.edu.cn; yantoji@163.com).}

\thispagestyle{empty}
\pagestyle{empty}

\begin{abstract}
A decentralized federated learning architecture is proposed to apply to the Businesses-to-Businesses scenarios by introducing the consortium blockchain in this paper. We introduce a model verification mechanism to ensure the quality of local models trained by participators. To analyze the latency of the system, a latency model is constructed by considering the work flow of the architecture. Finally the experiment results show that our latency model does well in quantifying the actual delays.
\end{abstract}

\begin{IEEEkeywords}
\textbf{Federated learning, consortium blockchain, model verification, latency.}
\end{IEEEkeywords}

\section{Introduction}

Quantities of data have been generated continuously and become the new type of fuel that promotes the development of production. But the data security has to be considered during training a model cooperatively among different participators. In this regard, federated learning (FL) architecture was proposed \cite{ref1}. 

The original FL system needs a central sever to collect and distribute the model weights and gradient parameters (called the local model update). This centralized architecture may introduce some problems. Firstly, the global model that all participators receive depends on the single central server. If a failure happens on the server, each participator would get an inaccurate global model. Secondly, because all the model updates are stored in the sever. Once the server is attacked, the whole system would be collapsed.

In order to avoid the negative effects brought by the centralized architecture, the decentralized architecture was proposed by exploiting the blockchain instead of the server \cite{ref5}. 

FL based on the blockchain has been used on Internet of Things (IoT) \cite{ref2}, Internet of Vehicular (IoV) \cite{ref3}, Mobile Edge Computing (MEC) \cite{ref4} and so on. It supports not only these Devices-to-Devices (D2D) applications but also Businesses-to-Businesses (B2B) scenarios. Enterprises that own mass of data, such as banks, securities and hospitals, would like to discover the intrinsic value hidden in the data collaborating with others. In this paper, we present a FL based on blockchain for these B2B scenarios.

Considering the efficiency of the architecture, consortium blockchain should be used for the decentralized federated learning \cite{ref6,ref7}. Because only authorized peers can join the network and have access to the data stored in the distributed ledger on the consortium blockchain. The consensus protocol of the consortium blockchain is usually not PoW (Proof of Work), but consensus algorithms such as PBFT (Practical Byzantine Fault Tolerance) \cite{ref8} and Raft \cite{ref9}, which are more efficient and suitable for the multi-center network. 

The verification mechanism of the blockchain is often used to authenticate identities of peers. But the quality of models is especially important in the FL. Thus we introduce a model verification mechanism in order to ensure the quality of local model updates trained by participators.

The efficiency of the blockchained federated learning system is a key issue for practical application. Therefore, it is important to analyse the latency of the system. Most of the existing works explained the system delay by the ways of numerical simulation. These empirical analyses are too costly to obtain accurate results. Furthermore, the underlying networks for deploying permissioned blockchains have a great impact on analysis results, thus these results are not comparable and lack versatility \cite{ref10}. It is imperative to analyse theoretical latency to provide a quantitative model. The main contributions of this paper are as follows:

\begin{itemize}
    \item A decentralized federated learning based on consortium blockchain called CBFL is proposed to train a classification model by logistic regression with a horizontally partitioned dataset in B2B scenarios.
    
    \item We introduce a model verification mechanism for CBFL to validate the availability of the local model updates trained by participators.
    
    \item The theoretic latency model for CBFL system is divided into three parts. Each part involves several subdivisions for fine-grained analysis. 
    \item Through the latency model, we get an optimal throughput configuration for PBFT so as to improve the efficiency in practical application.
\end{itemize}


\section{CBFL Architecture and Operation}
Let $E=\{E_{i}\}_{i=1}^{N_{E}}$ be a set of enterprises collaborating with each other in CBFL. The enterprises manage two types of nodes: compute nodes $\{C_{i}\}_{i=1}^{N_{E}}$ and communication peers $\{P_{i}\}_{i=1}^{N_{E}}$. The compute nodes have enough computing power to train the models. And communication nodes are responsible for maintaining the blockchain.

The CBFL architecture is organized as two layers: the model update layer and the blockchain layer as shown in Fig. 1. In the model update layer, compute nodes train the local models using its own raw data samples locally and upload the local model updates to the corresponding communication peers in the blockchain layer. Each peer verifies all the local model updates gathered from other peers and operates the consensus algorithm to generate a new block. Finally, the local model updates recorded in the newest block are aggregated locally by each compute node. So all the participators achieve data collaboration without leaking the raw data.

\begin{figure}
\centering
\includegraphics[height=5.2cm,width=8.7cm]{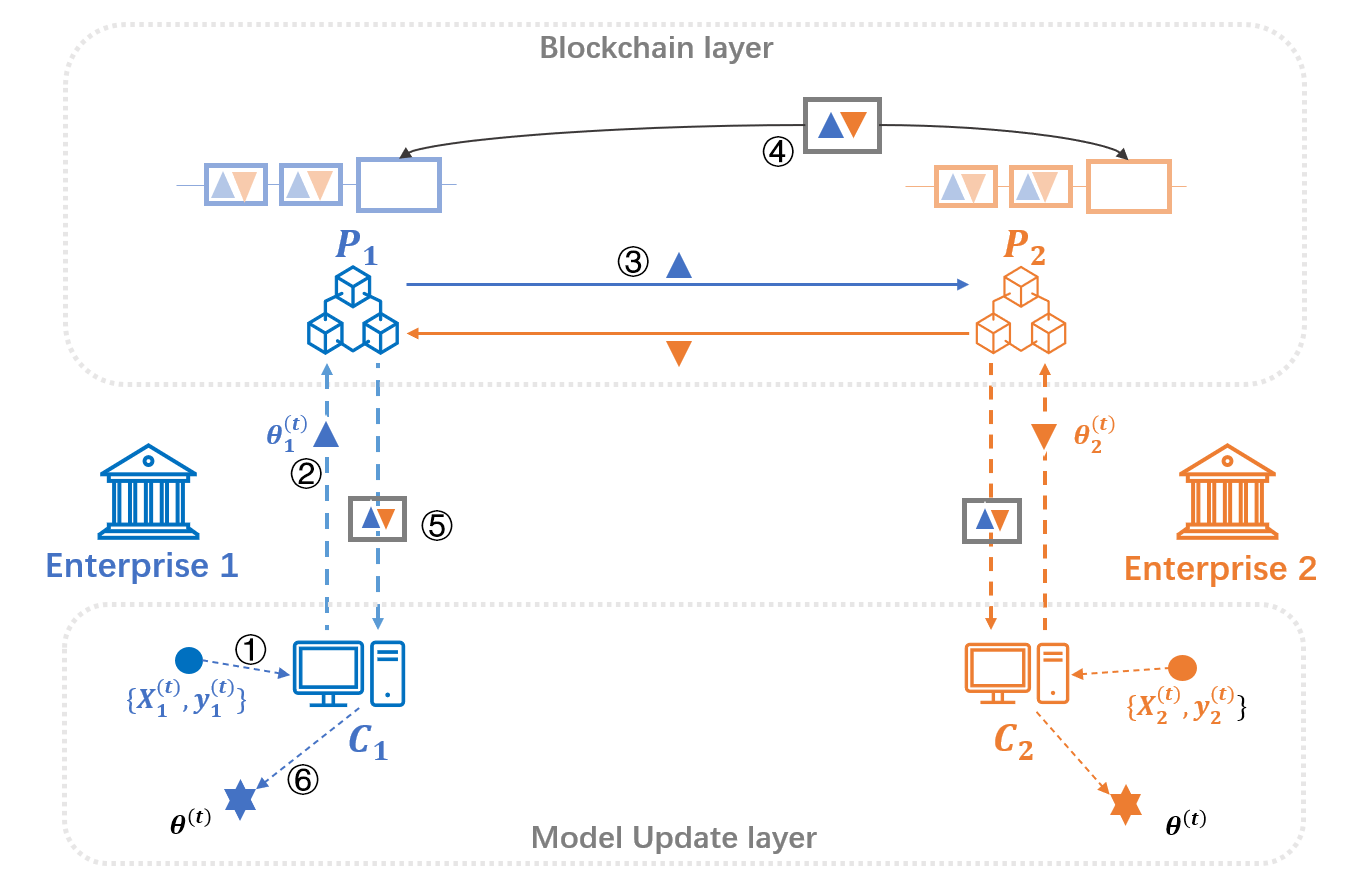}
\centering
\caption{CBFL Architecture.}
\end{figure}

\subsection{Model Update Layer}

Suppose that the $i$-th enterprise $E_i$ owns a set of data $D_i$ which includes $n$ features and $N_i$ samples, where $i \in \{1, 2, \cdots,N_{E}\}$. Let $D=\bigcup_{i=1}^{N_{E}}D_{i}$ be the entire dataset of all enterprises in CBFL, where $|D|=N_D=\sum_{i=1}^{N_E} N_{i}$. 

Our CBFL architecture focuses on the classification problem by using the logistic regression with the horizontally partitioned data \cite{ref11}. Let $\left\{x_{k}, y_{k}\right\} \in D_{i}$ be the $k$-th data sample, where $x_{k} \in \mathbb{R}^{n}$ and $y_{k}\in\{-1,1\}$. The goal of logistic regression is to train a linear model for classification by solving the following optimization problem \cite{ref12}:

\begin{equation}\label{(1)}
\min \frac{1}{~N_{D}} \sum_{i=1}^{N_{E}} \sum_{k=1}^{N_{i}} f_{k}\left(\omega ; x_{k}, y_{k}\right),
\end{equation}
where $\omega$ is the model parameter vector and
$$
f_{k}(\omega) \triangleq f_{k}\left(\omega ; x_{k}, y_{k}\right)=\log \left(1+\exp \left(y_{k} \cdot \omega^{T} x_{k}\right)\right).
$$

In order to solve the optimization problem (1), the model is locally trained with the stochastic variance reduced gradient(SVRG) \cite{ref1}:
\begin{small}
\begin{equation}
w_{i}^{t, \ell} \! = \! w_{i} ^ { t \! - \! 1, \ell } \! - \! \frac { \beta }{ N_{i} } \left( \left[ \nabla f_{k} \left( w_{i} ^ { t \! - \! 1, \ell } \right) \! - \! \nabla f_{k} \left( w ^ {\ell} \right) \right] \! + \! \nabla f \left( w^{\ell} \right) \right) \! , \!
\end{equation}
\end{small}where $\omega_{i}^{t, \ell}\in\mathbb{R}^{n}$ is the local weight at the $t$-th iteration of the $\ell$-th cycle and $\eta_{t}>0$ is the step size. Let $\omega_{i}^{\ell}$ be the local weight after the last local iteration of the $\ell$-th cycle. So $C_{i}$ gets the local model update $\left\{ \omega_{i}^{\ell },\left\{ \nabla f_{k} \left( \omega^{\ell} \right) \right\} \right\} \triangleq tx$. Then $C_{i}$ aggregates all the $tx$s to get the global model by
\begin{equation}
\omega^{\ell} = \omega^{\ell-1} + \sum_{i=1}^{N_{E}} \frac{N_{i}}{N_{D}} \left( \omega_{i}^{\ell} - \omega^{\ell-1} \right),
\end{equation}
where $\omega^{\ell}$ is the global model weight of the $\ell$-th cycle.

\subsection{Blockchain Layer with Model Verification Mechanism}

In the blockchain layer, the size of each block is set as $h+\delta_{m} N_{B}$, where $h$ is the block header size, $\delta_{m}$ is the single $tx$ size and $N^{B}$ is the maximum number of $tx$s within a block. It is more efficient to get the consensus by using a consortium blockchain instead of a public blockchain. Besides, peers in a consortium blockchain are authorized, data stored in the block are more secure.

The consensus protocol is the core component of a blockchain. In this paper, PBFT is adopted to get the consensus for the consortium blockchain network. It can get the consensus among $N_{P}$ peers with $f$ faulty peers, where $f=\frac{N_{P}-1}{3}$. 

PBFT includes three phases as shown in Fig. 2. A leader $L$ was chosen among all peers beforehand to create a candidate block in which $tx$s are sorted by timestamp. Then $L$ disseminates the candidate block to all other peers in a pre-prepare message at the pre-prepare stage. If a peer receives and accepts the message, it stores the message and enters the prepare phase broadcasting prepare messages. Then peers wait for a quorum of prepare messages, i.e., at least $2f+1$ prepare messages which match the stored pre-prepare message. In the third phase peers broadcast commit messages to all others. Then, if a peer collects another quorum of commit messages which match the previously collected prepare messages, it will commit the state transition and reply to the compute node \cite{ref8}. 

\begin{figure}
\centering\includegraphics[height=3.6cm,width=8.5cm]{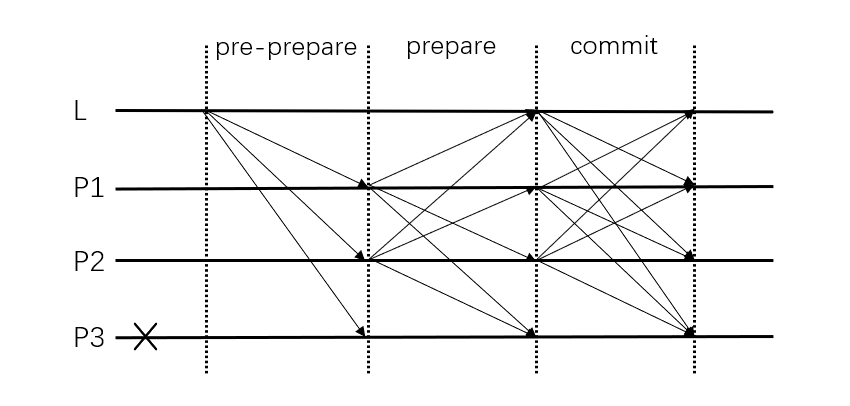}
\caption{Three phases of PBFT.}
\end{figure}

Replying on the blockchain network, we can build a secure data sharing platform where enterprises can exchange model updates to achieve secure data collaboration. Compute nodes can get all the local updates from the newest block and aggregate locally instead of downloading the global model update from the central server, which is more robust than the centralized FL. In the B2B application scenario, the number of peers is not too many. So PBFT can achieve the consensus efficiently. Furthermore, PBFT needs fewer computing resources than PoW and can avoid forking \cite{ref13}.

In a normal blockchain, peers verify the validity of a transaction with the digital signature technology. But with the federal learning protection privacy mechanism, some unreal or even malicious participants[14] will provide mendacious local model updates with some made-up data, which causes trouble for the global model. In our CBFL, the communication peers verify the $tx$s not only by checking the digital signatures but also by verifying the quality of models. 

We leverage the classification accuracy to quantify the performance of the local model updates. More specifically, the accuracy is denoted by the proportion of correctly classified samples. Denote that each $E_{i}$ owns $T$ testing instances for quantifying the accuracy of the local model updates. The classification accuracy $e_{j}$ of the $j$-th received local model update can be given by $e_{j}=\frac{n_{j}}{T}$, where $n_{j}$ is the number of the correctly classified samples. When the communication peer $P_{i}$ receives local model updates from other peers, it would admit the $j$-th local model update whose classification accuracy satisfies $e_{j} \geq e_{0}$, where $e_{0}$ is the threshold predetermined.

With the model verification mechanism, only the models trained by truthful data can be recorded in the distributed ledger of the blockchain. In this way, unnecessary oscillations can be avoided in the process of training the global model, which improves the efficiency of the whole system by reducing the training rounds.

\subsection{One-Cycle CBFL Operation}

As depicted in Fig. 1, the CBFL operation can be described by the following six steps:
\begin{enumerate}[Step 1]
\item Local model update: Each $C_{i}$ computes (2) with its own data to get the local model update $tx$.
\item Local model upload: $C_{i}$ uploads the $tx$ to its corresponding $P_{i}$.
\item Cross-verification: $P_{i}$ broadcasts the $tx$ obtained from $C_{i}$. At the same time, $P_{i}$ verifies the $tx$s received from other peers with our model verification mechanism. 
\item Consensus: The verified $tx$s are recorded in the candidate block by the leader $L$. The candidate block doesn't generate until reaching the block size $h+\delta_{m} N_{B}$ or maximum waiting time $\tau$. The leader $L$ multicasts the candidate block to all peers to start the three-phase PBFT to get the consensus among all peers.
\item Global model download: When a peer $P_{i}$ receives $2f+1$ commit messages, it sends the newest block which stores all participators’ $tx$s to the corresponding $C_{i}$ as the reply.
\item Global model update: Every $C_{i}$ computes the global model update by using (3) with all $tx$s recorded in the block.
\end{enumerate}
Step1 to Step6 is the one-cycle process of CBFL. This operation process doesn’t stop until $\left|\omega^{\ell}-\omega^{\ell-1}\right| \leq \varepsilon$.

\section{One-Cycle Operation Latency Analysis}
We aim to build a latency analysis model to quantify the time consumption of the CBFL. Before  building the latency model, some reasonable assumptions are made as follows:
\begin{itemize}
\item The compute nodes and communication peers have stable and enough computing resources for model training and verification.
\item The communication peers have certain communication and storage capabilities to ensure $tx$s sharing. And peers are defined to dispose the received messages on a FIFO basis, while the processing time at each peer follows the exponential distribution with the mean $\mu$.
\item The arrival of new $tx$s follows the Poisson Process with the arrival rate $\lambda$.
\end{itemize}
Let $T_{0}^{\ell}$ be the total time during $\ell$-th cycle process at a fixed enterprise $E_{0}$ and $$T_{0}^{\ell} = T_{update}+ T_{commun}+ T_{consensus},$$where $T_{update}$, $T_{consensus}$ and $T_{commun}$ are  model update, consensus and communication delays respectively.

1) Model update latency: The model update delays are generated by Step 1 and Step 6. Let $\delta_{d}$ be a single data sample size and $f_{c}$ be the clock speed. So the local model update latency in Step 1 is evaluated as $T_{local,0}^{\ell}=\delta_{d} N_{i} / f_{c}$ \cite{ref5}. And the global model update latency $T_{global,0}^{\ell}$ in Step 6 can be given as $T_{global,0}^{\ell}=\delta_{m} N_{B} / f_{c}$ \cite{ref5}, where $\delta_{m}$ is the size of a local model update $tx$.The model update latency can be calculated by
\begin{align}
T_{update} = T_{local,0}^{\ell} + T_{global,0}^{\ell}.
\end{align}

2) Consensus latency: The consensus delays are brought by Step 3 and Step 4. And the latency of the PBFT consensus is fully considered according to its three-phase work flow.

Let $N(\tau)$ be the number of arrived $tx$s within the max waiting time $\tau$. So the leader $L$ sends the pre-prepare message to other peers when the number of arrived $tx$s reaches $b$ according to the conditions above-mentioned in Step3, where $b=max\{ N(\tau), N_{B}\}$. The collection, verification and batch processes of $tx$s at $L$ can be modeled by the $ M/M/1 $ queue. According to the Little's law, The average waiting time of each $tx$ can be formulated as $\frac{1}{\mu - \lambda}$. Thus, the total latency of the pre-prepare phase can be given as
\begin{align}
    T_{preprepare} = \frac{max\{ N(\tau), N_{B}\}}{\mu - \lambda}.\notag
\end{align}

For an arbitrary fixed $P_{o}$, its process of receiving prepare messages is the Poisson process with the intensity $\lambda$. Thus, the time lag $t_{i}$ between two adjacent prepare messages follows the exponential distribution with mean $1 / \lambda$. The average waiting time of $P_{o}$ can be denoted as 
\begin{align}
T_{wait}=E[\sum_{i=1}^{2f} t_{i}]=\sum_{i=1}^{2f} E\left[t_{i}\right]=\frac{2f}{\lambda}.\notag
\end{align}
The total processing time in this phase is calculated as
\begin{align}
T_{process} = \frac{2f+1}{ \mu },\notag
\end{align}
so the latency of prepare phase is $$ T_{prepare} = T_{wait}+T_{process}.$$
The latency of the commit phase is similar to the prepare delay.

The total latency of consensus phase is 
\begin{align}
T_{consensus} = T_{preprepare} + T_{prepare} + T_{commit}.
\end{align}
\begin{figure}
\centering\includegraphics[scale = 0.3]{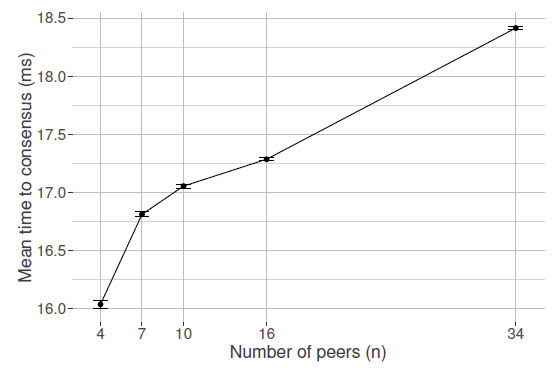}
\centering
\caption{Mean time to consensus for large number of peers.}
\end{figure}
3) Communication latency: The communication delays are contributed by Step 2 and Step 5. The local model upload latency in Step 2 is computed as $$T_{up,0}^{\ell}=\delta_{m} /\left[W_{up} \log _{2}\left(1+\gamma_{up}\right)\right],$$ where $W_{up}$ is the bandwidth between $C_{0}$ and $P_{0}$, $\gamma_{up}$ is the signal-to-noise ratio \cite{ref5}. Similarly, global model download delay in Step 5 is calculated as $$T_{dn,0}^{\ell}=\left(h+b \delta_{m}\right) /\left[W_{d n} \log _{2}\left(1+\gamma_{d n}\right)\right].$$ So the latency of communication can be calculated as
\begin{align}
T_{commun} = T_{up,0}^{\ell} + T_{dn,0}^{\ell}.
\end{align}

\begin{theorem}
If the algorithms for training local model updates and global model updates are confirmed, $T_{local,0}^{\ell}$ and  $T_{glocal,0}^{\ell}$ are constants.
$T_{up,0}^{\ell}$ and $T_{dn,0}^{\ell}$ are also constants when the underlying network is determined. Thus, the total latency of CBFL can be modeled as
\begin{align}
T_{0}^{\ell} &= T_{update}+ T_{commun}+ T_{consensus}\notag\\
&=T_{constant}+\frac{(b-4f)\lambda + 4f\mu}{\lambda (\mu - \lambda)} + \frac{4f+2}{\mu}.
\end{align}
\end{theorem}

\begin{proof}
According to the work flow of CBFL in Section \uppercase\expandafter{\romannumeral2}, $T_{0}^{\ell}$ is the sum of $T_{update}$, $T_{commun}$ and $T_{consensus}$. Let $T_{constant}= T_{update}+ T_{commun}$. And
\begin{align}
T_{consensus} &= T_{preprepare} + T_{prepare} + T_{commit}\notag\\
              &= \frac{b}{\mu - \lambda} + 2(\frac{2f}{\lambda}+\frac{2f+1}{\mu})\notag\\
              &= \frac{(b-4f)\lambda + 4f\mu}{\lambda (\mu - \lambda)} + \frac{4f+2}{\mu}.
\end{align}
\end{proof} 

\begin{theorem}
With the case where the leader starts the PBFT when the maximum size of a block is satisfied, i.e. $b =  N_{B}$, the optimal $\lambda ^{*}$ for PBFT can be given by
\begin{align}
    \lambda ^{*} = \frac{-8f\mu + 4\mu \sqrt{f N_{B} }} { 2 ( N_{B} - 4f )}.
\end{align}
\end{theorem}

\begin{proof} 
According to (8), we can get the first derivative and the second derivative of $T_{consensus}$ with respect to $\lambda$ as follows
\begin{align}
T_{consensus}^{'} &= \frac{(N_{B}-4f)\lambda^{2} + 8f\mu\lambda - 4f\mu^{2}}{\lambda^{2} (\mu - \lambda)^{2}}.\notag\\
T_{consensus}^{''} &= \frac{2N_{B}}{(\mu - \lambda)^{3}}+\frac{8f}{(\lambda)^{3}}.\notag
\end{align}
Thus the $T_{consensus}$ is convex for $\lambda $. The optimum $\lambda ^{*}$ is directly derived.
\end{proof} 

\section{Numerical Results and Conclusion}
\begin{figure}
\centering
\subfigure[Latency with varying $\lambda$]{
\begin{minipage}[t]{0.5\linewidth}
\centering
\includegraphics[width=1.8in]{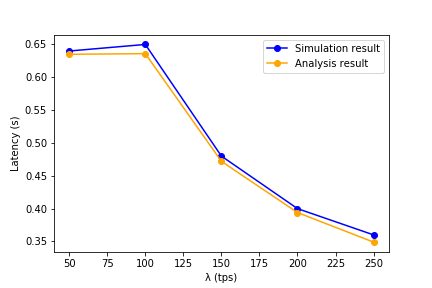}
\end{minipage}
}%
\subfigure[Latency results compare]{
\begin{minipage}[t]{0.5\linewidth}
\centering
\includegraphics[width=1.8in]{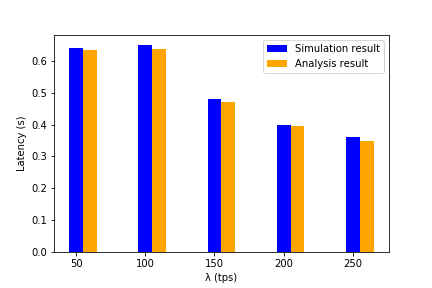}
\end{minipage}
}%
\centering
\caption{Blockchain latency for the transactions arrival rate $\lambda$ }
\end{figure}

The time consumption data was fitted with Exponential, Weibull, Gamma, Hypoexponential, LogNormal and Pareto distributions using MLE (Maximum Likelihood Estimation) technique in \cite{ref14}. The best-fit model of $T_{prepare}$ is Weibull distribution ($shape = 2.092$, $scale = 0.8468$). For another, $T_{prepare}$ is the sum of independent identically exponential distributed random variables given in Section \uppercase\expandafter{\romannumeral3}. Thus it follows Gamma distribution. While the Gamma distribution is a special kind of Weibull distribution, our latency analyses of $T_{prepare}$ and $T_{commit}$ are suitable.

As the number of peers increases, the probability of failure peers occurrence will also increase. In Fig. 3 \cite{ref14}, the mean latency to consensus increases with the augment of the number of peers. This is similar to what our model (7) shows. 

Fig. 4 shows the impact of the transactions arrival rate $\lambda$ on the blockchain average completion latency. The relationship between the latency and $\lambda$ is a approximate inverse proportional function as shown in Fig. 4-a, which is consistent with our latency model. In Fig. 4-b, we compare the latency results from the simulation \cite{ref10} and the latency model. In order to ensure comparability, the same configurations in \cite{ref10} are adopted. Let $N_{B}=100$, $N_{P}=4, f=1$ and the transactions arrival rate $\lambda$ starts from 50tps to 250tps in this experiment. The model predicts the experimental measurements with an error lower than 3.1\%. 

In the B2B scenarios, each enterprise can enhance the computing and communication equipment to improve model update and communication efficiency. So it is crucial to optimize CBFL with respect to latency, computing and storage requirements by improving the underlying networks, which can reduce the consistent delays according to our latency analysis model.

In conclusion, the consensus latency is a bottleneck of the whole system, especially the latency of waiting for ordering. According to (7), the latency of the system is proportional to the number of faulty nodes and inversely proportional to the system TPS. The throughput can be adjusted to reduce latency according to actual situation. In addition, faulty nodes can be reduced by establishing a reward system and a node selection mechanism.

\section{Acknowledgement}
The authors thank professors Debiao He of Wuhan University and Xiaohong Huang of Beijing University of Posts and telecommunications for their valuable suggestions to improve the  the innovation of this paper.

\balance
\end{document}